\documentclass[10pt,journal,compsoc]{IEEEtran}
\usepackage{cite}
\usepackage{url}
\usepackage{amsmath,amssymb,amsfonts}
\usepackage{graphicx}
\usepackage{textcomp}
\usepackage{epstopdf}
\usepackage{algorithm}
\usepackage{algpseudocode}
\usepackage{booktabs}
\usepackage{cases}
\usepackage{array}
\usepackage{caption}
\usepackage{subfigure}
\usepackage{verbatim}
\usepackage{multirow} 
\usepackage{bm}
\usepackage{float}
\usepackage{tablefootnote} 
\usepackage{empheq}
\usepackage{amsthm}

\newtheorem{lemma}{Lemma}

\begin{document}

\title{Subspace Nonnegative Matrix Factorization for Feature Representation}
\author{Junhang~Li, Jiao~Wei, Can~Tong, Tingting~Shen, Yuchen~Liu, Chen~Li, Shouliang~Qi, Yudong~Yao,~\IEEEmembership{Fellow,~IEEE,}~and~Yueyang Teng$^*$%

\thanks{\indent This work was supported by the Fundamental Research Fund for the Central Universities of China (N180719020).}%
       
\IEEEcompsocitemizethanks{\IEEEcompsocthanksitem J. Li, J. Wei, C. Tong, T. Shen, Y. Liu, C. Li and S. Qi are with the College of Medicine and Biological Information Engineering, Northeastern University, Shenyang 110169, China.
\IEEEcompsocthanksitem Y. Yao is with the Department of Electrical and Computer Engineering, Stevens Institute of Technology, Hoboken, NJ 07030, USA.
\IEEEcompsocthanksitem Y. Teng is with the College of Medicine and Biological Information Engineering, Northeastern University, Shenyang 110169, China; and the Key Laboratory of Intelligent Computing in Medical Image, Ministry of Education, Shenyang 110169, China (email: tengyy@bmie.neu.edu.cn).}}


\IEEEtitleabstractindextext{%
\begin{abstract}
Traditional nonnegative matrix factorization (NMF) learns a new feature representation on the whole data space, which means treating all features equally. However, a subspace is often sufficient for accurate representation in practical applications, and redundant features can be invalid or even harmful.
For example, if a camera has some sensors destroyed, then the corresponding pixels in the photos from this camera are not helpful to identify the content, which means only the subspace consisting of remaining pixels is worthy of attention.
This paper proposes a new NMF method by introducing adaptive weights to identify key features in the original space so that only a subspace involves generating the new representation. Two strategies are proposed to achieve this: the fuzzier weighted technique and entropy regularized weighted technique, both of which result in an iterative solution with a simple form. Experimental results on several real-world datasets demonstrated that the proposed methods can generate a more accurate feature representation than existing methods. The code developed in this study is available at https://github.com/WNMF1/FWNMF-ERWNMF.
\end{abstract}

}

\maketitle
\IEEEraisesectionheading{\section{Introduction}\label{sec:introduction}}
\IEEEPARstart{T}{oday}, various kinds of data are involved in all aspects of people's lives, and the analysis and processing of large-scale data are occupying an increasingly important position in the field of scientific research. However, as the dimension and amount of data increases, an increasing computational load is required to process them, and even redundant information may do harm to the analysis result. Therefore, dimensionality reduction techniques \cite{van2009dimensionality} are often used to overcome the problem of the curse of dimensionality. These techniques include principal component analysis (PCA) \cite{2002Principal}, isometric mapping (ISOMAP) \cite{tenenbaum2000global}, local linear embedding (LLE) \cite{roweis2000nonlinear}, Laplacian eigenmaps \cite{belkin2003laplacian} and nonnegative matrix factorization (NMF) \cite{lee1999learning}. 

Among them, NMF has good interpretability, dramatically satisfies real-world needs and has become one of the popular methods for dimensionality reduction. Lee and Seung \cite{wang2012nonnegative, ma2017nonnegative,lee2000algorithms} replaced the original high-dimensional nonnegative matrix with the product of two low-dimensional nonnegative matrices, which are called the base matrix and representation matrix. The nonnegative constraints result in a partial rather than global representation. The researchers proposed two objective functions based on the Euclidean distance and Kullback-Leibler divergence as the similarity measure and then optimized them by a simple multiplicative update rule with alternating optimization. 

After that, many NMF variants were proposed targeting different tasks. For example, Babaee \emph{et al.} \cite{babaee2016discriminative} proposed constrained NMF (CNMF) to improve the performance of the algorithm by label information to restrict the base matrix. Li \emph{et al.} \cite{liu2010non} proposed the local NMF (LNMF), which adds sparse constraints to the base matrix to preserve the important information in the base matrix. Hoyer \emph{et al.} \cite{hoyer2004non} proposed a sparse NMF (SNMF) to obtain a sparser representation by adding sparse encoding. Peharz \emph{et al.}\cite{peharz2012sparse} proposed a generic alternating update scheme for $\ell_0$-sparse NMF, which accelerates the computation and makes the results more accurate. The graph regularized NMF was proposed by using the nearest neighbor graph to explore the intrinsic geometric structure of the data \cite{cai2010graph,wang2013multiple}. This approach ensures that the data have as much as possible the same intrinsic structure in the low- and high-dimensional spaces. Choi \emph{et al.} \cite{choi2008algorithms} presented a simple algorithm for orthogonal NMF (ONMF), where orthogonality constraints are imposed on the base or representation matrix. Ding \emph{et al.} \cite{ding2006orthogonal} proposed orthogonal 3-factor factorization, which makes the learned representation more accurate. Subsequently, Yoo \emph{et al.} \cite{yoo2010orthogonal} proposed orthogonal nonnegative matrix triple factorization (NMTF) directly using the real gradient information optimization algorithm on Stiefel manifolds. Ding \emph{et al.} \cite{ding2008convex} proposed the SemiNMF method, which allows the existence of negative elements in the original and base matrices, thus making NMF widely available. They also restricted the base vectors to convex combinations of data points to develop ConvexNMF.

Many researchers have found that some redundant features impair representation, meaning that only a subspace of the original data is interesting, so they introduced weighted NMF (WNMF). Guillamet \emph{et al.} \cite{guillamet2003introducing} proposed a WNMF method by precalculating the weights of different types of samples in the original data to solve the problem caused by sample imbalance. Mao \emph{et al.} \cite{mao2004modeling} proposed a WNMF that incorporates binary weights into NMF multiplicative updates for handling missing values in distance matrices. After that, Kim \emph{et al.} \cite{kim2009weighted} developed two relatively fast and scalable algorithms for WNMF: alternating nonnegative least squares (ANLS-WNMF) and generalized expectation maximization (GEM-WNMF). Lu \emph{et al.} \cite{lu2009doubly} proposed a new weighting method called double WNMF, which utilizes two weighting matrices, between-sample and within-sample weighting matrices, to better exploit samples' discriminative and geometric information for imbalanced face recognition. Gao \emph{et al.} \cite{gao2016minimum} proposed a minimum-volume-regularized weighted SNMF (MV-WSNMF) based on the relationship between robust NMF and SNMF. The method can flexibly approximate the similarity matrix by introducing a weight matrix, thus making the obtained performance more robust to noise. 

These methods can solve the dimensionality redundancy problem well by using predefined weights, which may be called "hard WNMF." However, prior knowledge of the weights is generally unavailable, so it is a significant task to develop "soft WNMF," which flexibly assigns weights that do not need to be determined in advance.

This paper provides an adaptive weighted method to pursue a subspace of the original data for NMF, which assigns a weight to each feature to indicate its importance. The weights can be explained as the probability of the contribution of a feature to NMF. We introduce two strategies to solve the problem: fuzzier weight and entropy regularized weight, both of which are inspired by the derivation of fuzzy c-means clustering \cite{zhou2016fuzzy}. The former uses power hyperparameters to smooth the weight distribution, and the latter utilizes the entropy
to penalize the weights. Then, the Lagrange multiplier method is used to obtain two solutions with a simple form. The proposed methods are compatible with many existing NMF methods including the previously mentioned ONMF\cite{choi2008algorithms}, SemiNMF\cite{ding2008convex}, and ConvexNMF\cite{ding2008convex}. Experiments are performed on several real-world datasets, and the results show the feasibility and effectiveness of the proposed methods.

\section{Methodology}
\subsection{Related works}
$\mathbf{Notations:}~$~For the given matrix $A$, the $(i, j)-th$ element is indicated as $A_{ij}$. $A^T$ denotes the transpose of $A$. The symbols $\otimes$ and $\oslash$ respectively denote item-by-item multiplication and division of
two matrices. $A\geq 0$ means that all elements of $A$ are equal to or larger than 0. Moreover, diag$(\cdot)$ generates a diagonal matrix from a vector.

Let $X\in R^{M\times N}$ be the given nonnegative
matrix in which each column is a data point. NMF will approximate it by the product of two low-dimensional nonnegative matrices $ S \in R^{M\times K}$ and $H\in R^{K \times N}$, where $K<<\min\{M,N\}$. We choose the square of the Euclidean distance as the metric, and the problem is formulated as follows:
\begin{equation}
\begin{aligned}
\  \min&~ F_1(S,H)=\sum_{i=1}^M \sum_{j=1}^N \left[X_{ij}-(SH)_{ij}\right]^2\\
s.~t.&~ S\geq 0,~ H \geq 0
\label{eq1}
\end{aligned}
\end{equation}

To alternatively minimize $S$ and $H$ in Eq. (\ref{eq1}), the construction of an auxiliary function is important to determine the iterative update rule.\\
$\mathbf{Definition~1} ~(Auxiliary~function)$ If the function $G(h,h')$ satisfies the following conditions:
\begin{equation}
G(h,h')\ge F(h)~~~and~~~G(h',h')=F(h')
\label{eq2} 
\end{equation}
where $h'$ is a given value, then $G(h,h')$ is the auxiliary function of $F(h)$ on $h'$. 

Thus, we can draw the following conclusion.\\
\begin{lemma} If $G(h,h')$ is an auxiliary function of $F(h)$, then under the update rule:
	\begin{equation}
	h^*=arg\mathop{\min}_{h}G(h,h')
	\label{eq3}
	\end{equation}
	the function $F(h)$ does not increase. \\
\end{lemma}
\begin{proof} The conditions satisfied by the auxiliary function make this proof marked because
	\begin{equation}
	F(h^{*})\leq G(h^{*},h')\leq G(h',h')\leq F(h')
	\label{eq4}
	\end{equation}
\end{proof}

Thus, if the auxiliary function reaches the minimum, then the original function does not increase. Next, we construct the update rule for NMF. 

Consider $S$ first, where $S^t>0$ represents the $t$-th iteration and $H>0$ is given. Let $\xi_{ijl}=S^t_{il}H_{lj}/(S^tH)_{ij}$, of course, $\xi_{ijl}\geq0$ and $\sum_{l=1}^K\xi_{ijl}=1$. Therefore, the auxiliary function of $F_1$ is
\begin{equation}
f_1(S,S^t)=\sum_{i=1}^M\sum_{j=1}^N\sum_{l=1}^K\xi_{ijl}(X_{ij}-\frac{S_{il}H_{lj}}{\xi_{ijl}})^2
\label{eq5}
\end{equation}

The variables $S$ in the objective function are separable to be minimized. We take the partial derivative of Eq. (\ref{eq5}) and set it to zero so that we can obtain the following update rule:
\begin{equation}
S \leftarrow S \otimes (XH^T)\oslash(SHH^T)
\label{eq6}
\end{equation}

This is similar to deriving the update rule to $H$ as
\begin{equation}
H \leftarrow H \otimes (S^TX)\oslash(S^TSH)
\label{eq7}
\end{equation}


\subsection{Proposed method}

To address the importance of the $i$-th feature in $X$, we introduce an optimizable weight $W_i$. Then, Eq. (\ref{eq1}) can be reformulated as
\begin{equation}
\begin{aligned}
\min&~F_2(W,S,H)=\sum_{i=1}^M \sum_{j=1}^N W_{i}[X_{ij}-(SH)_{ij}]^2\\ 
s.~t.&~ S\geq0,~H\geq0,~W\geq0,~\sum_{i=1}^M W_{i}=1
\end{aligned}\label{eq8}
\end{equation}

The new variable $W$ can be solved in an alternative optimization manner. However, for fixed $S$ and $H$, $W_{i}$ is very easy to solve as
\begin{eqnarray}
&&W_i=
\begin{cases}
1& \text{$E_{i}=min\{E_{1}, E_{2},..., E_{M}\}$}\\
	0& \text{else}
\end{cases}\nonumber\\
&&where~~E_i=\sum_{j=1}^N [X_{ij}-(SH)_{ij}]^2\label{eq:minW}
\end{eqnarray}

The solution can be explained by a simple example, which is shown in the Appendix. This process is similar to the computation of the weights in the k-means clustering algorithm \cite{macqueen1967some}. Furthermore, it is indeed incompatible with the real problem. A natural concept is that the weight should fall in the range of $[0~1]$ instead of 0 or 1. As follows, we apply two techniques to solve this issue.

\subsubsection{Fuzzier WNMF (FWNMF)}

First, we introduce a hyperparameter $p>1$ as a given power of the weight $W_i$ so that $W_i$ can be seen as a fuzzier (value in the range of $[0~1]$). We then obtain the following model:
\begin{equation}
\begin{aligned}
\min&~F_3(W, S, H)=\sum_{i=1}^M \sum_{j=1}^N   W_{i}^p\,[ X_{ij}-(SH)_{ij}]^2\\ 
s.~t.&~ S\geq0,~H\geq0,~W\geq0,~\sum_{i=1}^M W_{i}=1
\label{eq9}
\end{aligned}
\end{equation}

We fix $S$ and $H$ and then solve $W$. The Lagrange function of Eq. (\ref{eq9}) can be constructed as
\begin{equation}
\begin{aligned}
L(W,\lambda)=\sum_{i=1}^M \sum_{j=1}^N W_{i}^p\,[X_{ij}-(SH)_{ij}]^2-\lambda\,(\sum_{i=1}^MW_{i}-1)
\label{eq10}
\end{aligned}
\end{equation}
where $\lambda$ is the Lagrange multiplier. By setting the partial derivative to $W_{i}$ and $\lambda$ to be zero, we obtain the following equation system:
\begin{numcases}{}
\frac{\partial L(W,\lambda)}{\partial W_{i}} &=$pW_{i}^{p-1}\sum\limits_{j=1}^N[X_{ij}-(SH)_{ij}]^2-\lambda=0$ \label{eq11}\\
\frac{\partial L(W,\lambda)}{\partial \lambda}& =$\sum_{i=1}^M W_i-1=0$
\label{eq12}
\end{numcases}

From Eq. (\ref{eq11}), we know that:\\
\begin{equation}
W_{i}=\sqrt[\leftroot{-2}\uproot{13}p-1]{\frac{\lambda}{p}} \sqrt[\leftroot{-2}\uproot{17}p-1]{\frac{1}{\sum_{j=1}^N[X_{ij}-(SH)_{ij}]^2}}
\label{eq13}
\end{equation}

Substituting Eq. (\ref{eq13}) into Eq. (\ref{eq12}), we have\\
\begin{equation}
\begin{aligned}
\sum_{i=1}^M W_i&=\sqrt[\leftroot{-2}\uproot{13}p-1]{\frac{\lambda}{p}}~~\sum_{i=1}^M\sqrt[\leftroot{-2}\uproot{17}p-1]{\frac{1}{\sum_{j=1}^N[X_{ij}-(SH)_{ij}]^2}}\\&=1
\label{eq14}
\end{aligned}
\end{equation}

Rearranging Eq. (\ref{eq14}), we have\\
\begin{equation}
\sqrt[\leftroot{-2}\uproot{13}p-1]{\frac{\lambda}{p}}=\frac{1}{\sum\limits_{i=1}^M \sqrt[\leftroot{-2}\uproot{13}p-1]{\frac{1}{\sum_{j=1}^N[X_{ij}-(SH)_{ij}]^2}}}
\label{eq15}
\end{equation}

Substituting this expression into Eq. (\ref{eq13}), we find that\\
\begin{equation}
W_{i}=\frac{\sqrt[\leftroot{-2}\uproot{7}p-1]{\frac{1}{\sum_{j=1}^N[X_{ij}-(SH)_{ij}]^2}}}{\sum\limits_{l=1}^M \sqrt[\leftroot{-2}\uproot{7}p-1]{\frac{1}{\sum_{j=1}^N[X_{lj}-(SH)_{lj}]^2}}} 
\label{eq16}
\end{equation}

Next, we can solve $S$ and $H$ with fixed $W$, which is similar to the standard NMF. For example, we can construct the following auxiliary function about $S$:\\
\begin{equation}
f_2(S,S^t)=\sum_{i=1}^M \sum_{j=1}^N\sum_{l=1}^K W_{i}^p\,\xi_{ijl}(X_{ij}-\frac{S_{il}H_{lj}}{\xi_{ijl}})^2
\label{eq17}
\end{equation}

Setting the partial derivative of $f_2(S,S^t)$ to zero yields the following update rule:\\
\begin{equation}
S \leftarrow S \otimes (XH^T)\oslash(SHH^T)
\label{eq18}
\end{equation}

Similarly, we can easily obtain the update rule for $H$ as follows:\\
\begin{equation}
\begin{aligned}
H \leftarrow H \otimes [(\overline{W^p} S)^TX]\oslash[(\overline{W^p} S)^TSH]
\label{eq19}
\end{aligned}
\end{equation}
where $\overline{W^p}=diag([W_1^p,\cdots,W_M^p])$. 

FWNMF is summarized in $\mathbf{Algorithm~ 1}$:
\begin{algorithm}[htb]
	\caption{ FWNMF}
	\label{alg:Framwork}
	\begin{algorithmic}[1]
		\Require
		Given input nonnegative matrix $X\subseteq R^{M \times N}$, reduced dimension number $K$, and hyperparameter $p$;
		\Ensure
		weight vector $W$, base matrix $S$, and representation matrix $H$;
		\State Randomly initialize $S\in R^{M\times K}> 0$ and $H\in R^{K \times N} > 0$;
		\While{not convergence}
		\label{code:fram:trainbase}
		\State Update $W$ by Eq. (\ref{eq16});
		\label{code:fram:add}
		\State Update $S$ by Eq. (\ref{eq18});
		\label{code:fram:classify}
		\State Update $H$ by Eq. (\ref{eq19});
		\label{code:fram:classify}
		\EndWhile
		\label{code:fram:select} \\
		\Return  $W$, $S$ and $H$.
	\end{algorithmic}
\end{algorithm}

\subsubsection{Entropy regularized WNMF (ERWNMF)}
We introduce the entropy regularizer into Eq. (\ref{eq8}), which represents the uncertainty of the weights. The problem is reformulated as follows:
\begin{equation}
\begin{aligned}
\min ~F_4( W, S,H)=&\sum_{i=1}^M \sum_{j=1}^N W_{i}[X_{ij}-(SH)_{ij}]^2\\
&+\gamma\,\sum_{i=1}^M  W_iln( W_{i})\\ 
s.~t.~S\geq0,~H\geq0,&~W\ge0,~\sum_{i=1}^MW_i=1
\label{eq20}
\end{aligned}
\end{equation}
$\gamma\ge0$ is a hyperparameter that controls the strength of the entropy regularizer. Minimizing it can motivate more dimensions to participate in factorization. 

We first solve $W$ as we just did. Let $\lambda$ still be the Lagrange multiplier, and we still construct a Lagrange function.
\begin{eqnarray}
L(W,\lambda)&=&\sum_{i=1}^M \sum_{j=1}^NW_{i}[X_{ij}-(SH)_{ij}]^2\nonumber\\
&&+\gamma\sum_{i=1}^MW_{i}ln(W_{i})-\lambda\,(\sum_{i=1}^MW_{i}-1)
\label{eq21}
\end{eqnarray}

By setting the gradient of Eq. (\ref{eq21}) with regard to $W_{i}$ and $\lambda$, we obtain the following equation system:

\begin{numcases}{}
	\frac{\partial L(W,\lambda)}{\partial W_{i}}=&$\sum\limits_{j=1}^N[X_{ij}-(SH)_{ij}]^2$\nonumber\\
	&$+\gamma\ln W_{i}+\gamma-\lambda=0$\label{eq22}\\
	\frac{\partial L(W,\lambda)}{\partial \lambda}=&$\sum\limits_{i=1}^MW_{i}-1=0$
	\label{eq23}
\end{numcases}{}

From Eq. (\ref{eq22}), we know that\\
\begin{equation}
\begin{aligned}
W_{i}=e^{\frac{\lambda-\gamma}{\gamma}}e^{-\frac{\sum_{j=1}^N[X_{ij}-\left(SH\right)_{ij}]^2}{\gamma}}
\label{eq24}
\end{aligned}
\end{equation}

Substituting Eq. (\ref{eq24}) into Eq. (\ref{eq23}), we have\\
\begin{equation}
\begin{aligned}
\sum\limits_{i=1}^MW_{i}=e^{\frac{\lambda-\gamma}{\gamma}}\sum\limits_{i=1}^Me^{-\frac{\sum_{j=1}^N[X_{ij}-\left(SH\right)_{ij}]^2}{\gamma}}=1
\label{eq25}
\end{aligned}
\end{equation}

Rearranging Eq. (\ref{eq25}), we have\\
\begin{eqnarray}
e^{\frac{\lambda-\gamma}{\gamma}} =\frac{1}{\sum\limits_{i=1}^M e^{{-\frac{\sum_{j=1}^N[X_{ij}-(SH)_{ij}]^2}{\gamma}}}}  
\label{eq26}
\end{eqnarray}

Substituting this expression into Eq. (\ref{eq24}), we find that\\
\begin{equation}
W_{i}=\frac{e^{-\frac{\sum_{j=1}^N[X_{ij}-(SH)_{ij}]^2}{\gamma}}}{\sum\limits_{l=1}^M e^{-\frac{\sum_{j=1}^N[X_{lj}-(SH)_{lj}]^2}{\gamma}}} 
\label{eq27}
\end{equation}

We can then update $S$ and $H$ in a similar way as above.
\begin{eqnarray}
S &\leftarrow& S \otimes (XH^T)\oslash(SHH^T)
\label{eq28}\\
H &\leftarrow& H \otimes [(\overline{W} \otimes S)^TX]
\oslash[(\overline{W} \otimes S)^TSH]
\label{eq29}
\end{eqnarray}
where $\overline{W}=diag([W_1,\cdots,W_M])$. 

The details of ERWNMF are shown in $\mathbf{Algorithm~ 2}$:
\begin{algorithm}[htb]
	\caption{ ERWNMF}
	\label{alg:Framwork1}
	\begin{algorithmic}[1]
		\Require
		Given the input nonnegative matrix $X\subseteq R^{M \times N}$, number
		$K$ of reduced dimensions, and hyperparameter $\gamma$;
		\Ensure
		weight vector $W$, base matrix $S$, and representation matrix $H$;
		\State Randomly initialize $S\in R^{M \times K}> 0$ and $H\in R^{K \times N} > 0$;
		\While{not convergence}
		\State Update $W$ by Eq. (\ref{eq27});
		\State Update $S$ by Eq. (\ref{eq28});
		\State Update $H$ by Eq. (\ref{eq29});
		\EndWhile\\
		\Return  $W$, $S$ and $H$.
	\end{algorithmic}
\end{algorithm}

The weights $W$ in both FWNMF and ERWNMF denote the weights corresponding to the features in the original data. The similarity lies in that both of them depend on the reconstructed error, $\sum_{j=1}^N[X_{ij}-(SH)_{ij}]^2$, corresponding to each feature. Less error means more contribution to NMF. The difference is that they use different functions to describe this behavior.

\section{Extensions}
These techniques can be easily extended to other NMF methods. We take ConvexNMF \cite{ding2008convex} as an example to construct two new subspace NMF methods. 

First, we extend the FWNMF technique to ConvexNMF and obtain the optimization model as
\begin{equation}
\begin{aligned}
&\min F_5(W,S,H)=\sum_{i=1}^M \sum_{j=1}^N  W_{i}^p\,[ X_{ij}-\left( XSH\right)_{ij}]^2\qquad\qquad\qquad\qquad\\ 
&s.~t.~ S\geq0,~H\geq0,~W\geq0,~\sum_{i=1}^MW_{i}=1
\label{eq30}
\end{aligned}
\end{equation}
The corresponding update rule is as follows:
\begin{eqnarray}
W_{i}&\leftarrow&\frac{\sqrt[\leftroot{-2}\uproot{7}p-1]{\frac{1}{\sum_{j=1}^N[X_{ij}-(XSH)_{ij}]^2}}}{\sum\limits_{l=1}^M \sqrt[\leftroot{-2}\uproot{7}p-1]{\frac{1}{\sum_{j=1}^N[X_{lj}-(XSH)_{lj}]^2}}} 
\label{eq31}\\
\nonumber\\
S &\leftarrow& S\otimes[(\overline{W^p}  X)^TXH^T] \nonumber\\
&&\oslash [(\overline{W^p}X)^T XSHH^T]  \label{eq32}\\
\nonumber\\
H &\leftarrow& H\otimes[(\overline{W^p} XS)^TX]\nonumber\\
  && \oslash [(\overline{W^p} XS)^TXSH]\label{eq33} 
\end{eqnarray}

We also extend the ERWNMF technique to ConvexNMF and obtain the following optimization model:
\begin{equation}
\begin{aligned}
&\min F_6(W,S,H)=\sum_{i=1}^M \sum_{j=1}^N  W_{i}[ X_{ij}-( X S H)_{ij}]^2\\
&~~~~~~~~~~~~~~~~~~~~~~~~~~~~~~~ +\gamma\sum_{i=1}^M W_{i}ln(W_{i})\\  
&s.~t.~ S\geq0,~H\geq0,~W\geq0,~\sum_{i=1}^MW_{i}=1
\label{eq34}
\end{aligned}
\end{equation}
The corresponding update rule is as follows:
\begin{eqnarray}
W_{i}&\leftarrow&\frac{e^{-\frac{\sum_{j=1}^N[X_{ij}-(XSH)_{ij}]^2}{\gamma}}}{\sum\limits_{l=1}^M e^{-\frac{\sum_{j=1}^N[X_{lj}-(XSH)_{lj}]^2}{\gamma}}} 
\label{eq35}\\
\nonumber\\
S &\leftarrow& S\otimes[(\overline{W}  X)^TXH^T] \nonumber\\
&&\oslash [(\overline{W}X)^T XSHH^T]  \label{eq32}\\
\nonumber\\
H &\leftarrow& H\otimes[(\overline{W} XS)^TX]\nonumber\\
&& \oslash [(\overline{W} XS)^TXSH]\label{eq37} 
\end{eqnarray}

Their detailed derivation is omitted because they are very similar to the above methods.
Overall, our technique can be well popularized to many existing methods.
\begin{figure}[H]
	\centering
	\subfigure[]{
		\includegraphics[width=0.45\textwidth]{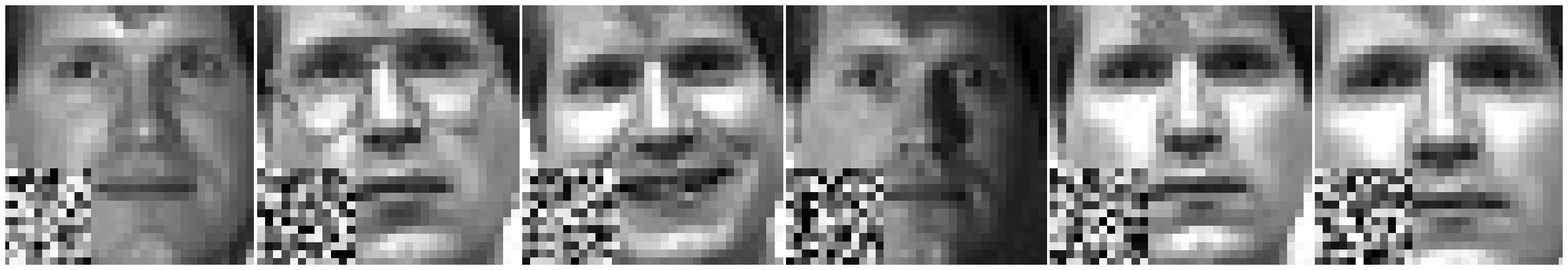} \label{1}
	}
	\quad
	\subfigure[]{
		\includegraphics[width=0.45\textwidth]{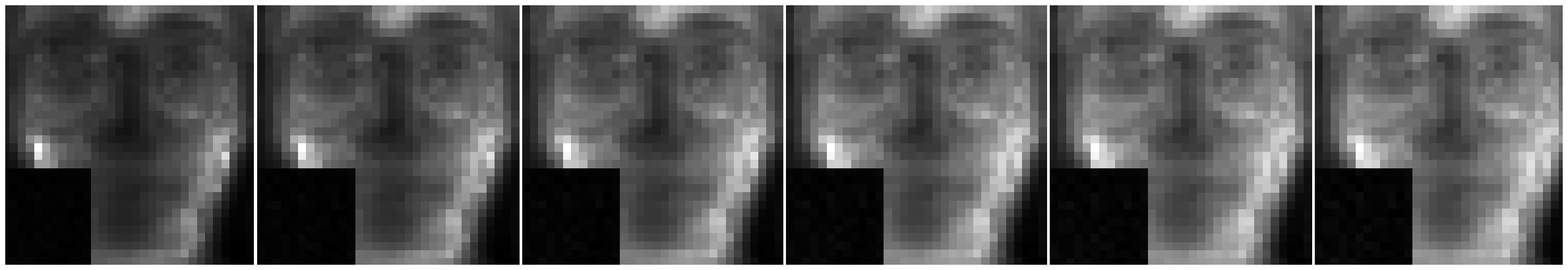} \label{2}
	}
	\quad
	\subfigure[]{
		\includegraphics[width=0.45\textwidth]{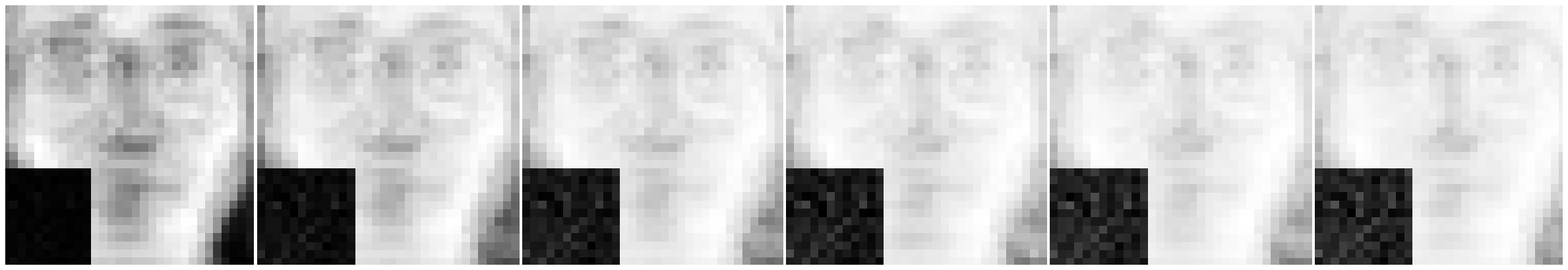}\label{3}
	}
	
	\caption{Weights $W$ obtained by FWNMF and ERWNMF: (a) six corrupted images of Yale dataset, (b) weight image obtained by FWNMF with $p$ ranging in $\{4, 4.5, 5, 5.5, 6, 6.5 \}$, and (c) weight image obtained by ERWNMF with $\gamma$ are chosen in order of $\{2^2,2^3,2^4,2^5,2^6,2^7\}$.}
	\label{fig.1}

\end{figure}

\section{Experiments}
\subsection{Experimental description}
\subsubsection{Setting and datasets}
To illustrate the effectiveness of our method in clustering tasks, we compare it with six existing methods: NMF\cite{lee1999learning}, ONMF\cite{choi2008algorithms}, NBVD\cite{long2005co}, NMTF\cite{yoo2010orthogonal}, SemiNMF\cite{ding2008convex}, and ConvexNMF\cite{ding2008convex}.
Experiments were performed on an HP Compaq PC with a 3.40-GHz Core i7-6700 CPU and 16 GB of memory, and all methods were implemented in MATLAB. 
All methods are initialized with the same uniform distribution for $W$ and $H$ over [0.1 1.1], and 300 iterations are performed to ensure sufficient convergence. We evaluate the proposed methods using real-world datasets, the details of which are listed in Table \ref{tab1}. We normalize each sample in the dataset in the range [0~1].

\begin{table}[!t] \caption{Description of datasets.}
	\label{tab1}	\centering
	\begin{tabular}{lcccccc}  
		\toprule   
		Dataset & Samples & Dimensions & Classes \\	
		\midrule   
		Yale\cite{belhumeur1997eigenfaces}&165&$32\times32$&15   \\  
		ORL\cite{kuang2015symnmf}&400&$32\times32$&40\\
		UMISTface\tablefootnote{http://www.sheffield.ac.uk/eee/research/iel/research/face/}&1012&$32\times32$&20\\
		GRIMACE\tablefootnote{https://cmp.felk.cvut.cz/~spacelib/faces/grimace.html}&360&$32\times32$&18\\
		COIL20\tablefootnote{https://www.cs.columbia.edu/CAVE/software/softlib/coil-20.php}&1440&$32\times32$&20\\
		GTFD\cite{chen2005face}&750&$32\times32$&50\\
		AR\tablefootnote{http://www2.ece.ohio-state.edu/~aleix/ARdatabase.html}&1400&$32\times32$&100\\
		Iris\tablefootnote{https://archive.ics.uci.edu/ml/datasets/Iris.}&300&4&3\\
		Brightdata\tablefootnote{http://pages.cs.wisc.edu/~olvi/}&2462&14&2\\
		Dimdata\footnotemark[6]&4192&14&2\\
		Satimage\tablefootnote{https://sci2s.ugr.es/keel/index.php}&6435&36&7\\
		Movement\footnotemark[7]&360&90&10\\
		\bottomrule  
	\end{tabular}
\end{table}

\begin{table*}[htbp]
	\renewcommand\arraystretch{1.5}
	\begin{center}
		\caption{Accuracy on real-world datasets. Best results are in boldface.}
		\begin{tabular}{lccccccccc} 
			\toprule 
			\multirow{2}{*}{Dataset}&\multicolumn{8}{c}{Accuracy} \\
			\cline{2-9}
			&NMF & ONMF & NBVD & NMTF & SemiNMF & ConvexNMF & FWNMF & ERWNMF \\
			\midrule
			Yale &0.3730& 0.3467& 0.3415& 0.3036& 0.3958&0.3152&0.3845& \textbf{0.4003} \\
			ORL&0.6178&0.5101& 0.4435& 0.4409& 0.5826& 0.2658&0.6233&\textbf{0.6325}\\
			UMISTface&0.4960& 0.4375& 0.4207& 0.3754& 0.4985& 0.2141&0.4910&\textbf{0.5045}\\ 
			GRIMACE&0.7406& 0.6622& 0.7232& 0.6651& 0.7154& 0.6331&0.7336&\textbf{0.7553}\\
			COIL20&0.5870& 0.5397& 0.5745& 0.5498& 0.5876& 0.5018&0.3598& \textbf{0.6017}\\ 
			GTFD&0.601 &0.5271& 0.5487& 0.5871& 0.5856& 0.4689&0.5984&\textbf{0.6055}\\
			AR&0.3166& 0.2401& 0.1606& 0.1672& 0.3178& 0.1383&0.2752&\textbf{0.3200}\\
			Iris&0.6957& 0.6972& 0.7060&0.6877& 0.7537& 0.6540&0.7417& \textbf{0.7672}\\
			Brightdata&0.7536&0.7340& 0.7319& 0.7010& 0.7427& 0.6720&0.6836&\textbf{ 0.8462}\\
			Dimdata& 0.5188&0.5485& 0.5240& 0.5931& 0.5114& 0.5355&0.7738&\textbf{0.8115}\\
			Satimage&0.7048& 0.5301& 0.5941& 0.5744& 0.5748& 0.5674&0.7148&\textbf{ 0.7212}\\
			Movement& 0.4531& 0.4119& 0.4188& 0.3638& 0.4036& 0.3306&0.458 &\textbf{ 0.4682}\\
			\specialrule{0em}{1pt}{1pt}
			\bottomrule  	
		\end{tabular} \label{tab2}
	\end{center}
\end{table*}

\begin{table*}[htbp]
	\renewcommand\arraystretch{1.5}
	\begin{center}
		\caption{NMI on real-world datasets. Best results are in boldface.}
		\begin{tabular}{lccccccccc} 
			\toprule 
			\multirow{2}{*}{Dataset} &\multicolumn{8}{c}{NMI} \\
			\cline{2-9}
			&NMF & ONMF & NBVD & NMTF & SemiNMF & ConvexNMF & FWNMF & ERWNMF \\
			\hline
			Yale &0.4363& 0.4176& 0.4147& 0.3770&\textbf{0.4633}& 0.3751& 0.4537& 0.4541 \\
			ORL &0.8166 &0.7429 &0.6834 &0.6754 &0.7838 &0.5134 &0.8191 &\textbf{0.8226}\\
			UMISTface&0.677 & 0.6180& 0.5865& 0.5202& 0.6719& 0.2645& 0.6636&\textbf{0.6806}\\ 
			GRIMACE&0.9120& 0.8582& 0.8969&0.8600&0.9030&0.8020& 0.9095& \textbf{0.9146}\\
			COIL20&0.7297& 0.7012& 0.7007& 0.6964&0.7283& 0.6307& 0.5009& \textbf{0.7377}\\
			GTFD&0.8437&0.7839& 0.8004& 0.8252&0.8374& 0.7242& \textbf{0.8494}&0.8474\\
			AR&0.6182& 0.5597& 0.4825& 0.4895&0.5821& 0.4541& 0.5965&\textbf{0.6210}\\
			Iris&0.6434& 0.5983& 0.5696& 0.6234&0.6265& 0.6129& 0.6180& \textbf{0.6649}\\
			Brightdata&0.1717& 0.1466& 0.1489& 0.1569&0.1540& 0.0905& 0.1279&\textbf{ 0.3862}\\
			Dimdata&0.0010& 0.0426& 0.0038& 0.0790&0.0004& 0.0060&0.3086&\textbf{ 0.3756}\\
			Satimage&0.6009& 0.4574&0.5236& 0.5044&0.5075& 0.4896&\textbf{ 0.6075}&0.6061\\
			Movement&0.581 & 0.5399&0.5343& 0.5065&0.4840& 0.4294& 0.5710& \textbf{0.5946}\\
			\specialrule{0em}{1pt}{1pt}
			\bottomrule  	
		\end{tabular} \label{tab3}
	\end{center}
\end{table*}

\subsubsection{Evaluation}
We used clustering accuracy \cite{liu2011constrained,plummer1986matching} and normalized mutual information (NMI)\cite{cai2005document,shahnaz2006document} to evaluate the clustering performance of all methods. A larger value indicates a better clustering effect.
 
 Given a set of actual labels $y$ and clustering results $y'$, the accuracy is defined as\\
\begin{equation}
	Accuracy=\frac{\sum_{i=1}^N\delta\bigl(y_i,map(y'_i)\bigr)}{N}
	\label{eq38}
\end{equation}
where:\\
$$\delta(x,y)=
\begin{cases}
	1& \text{\emph{x = y}}\\
	0& \text{otherwise}
\end{cases}$$
$N$ denotes the total number of samples in the dataset, and $map(\cdot)$ is a permutation mapping function that maps the obtained cluster labels to the actual labels.

NMI is defined as\\
\begin{equation}
NMI(y,y')=\frac{MI(y,y')}{max\bigl(H(y),H(y')\bigr)}
\label{eq39}
\end{equation}
where $H(y)$ is the entropy of $y$. $MI(y,y')$ quantifies the amount of information between $y$ and $y'$ and is defined as

\begin{equation}
MI(y,y')=\sum_{y_i\in y,y'_j\in y'}p(y_i,y'_j)\,log\bigl(\frac{p(y_i,y'_j)}{p(y_i)p(y'_j)}\bigr)
\label{eq40}
\end{equation}
where $p(y_i)$ and $p(y'_j)$ are the probabilities that the data points belong to clusters $y_i$ and $y'_j$, respectively; and $p(y_i,y'_j)$ is the joint probability that an arbitrarily selected data point belongs to clusters $y_i$ and $y'_j$ concurrently. 

\begin{figure*}[htbp]
	\centering
	\includegraphics[scale=0.45]{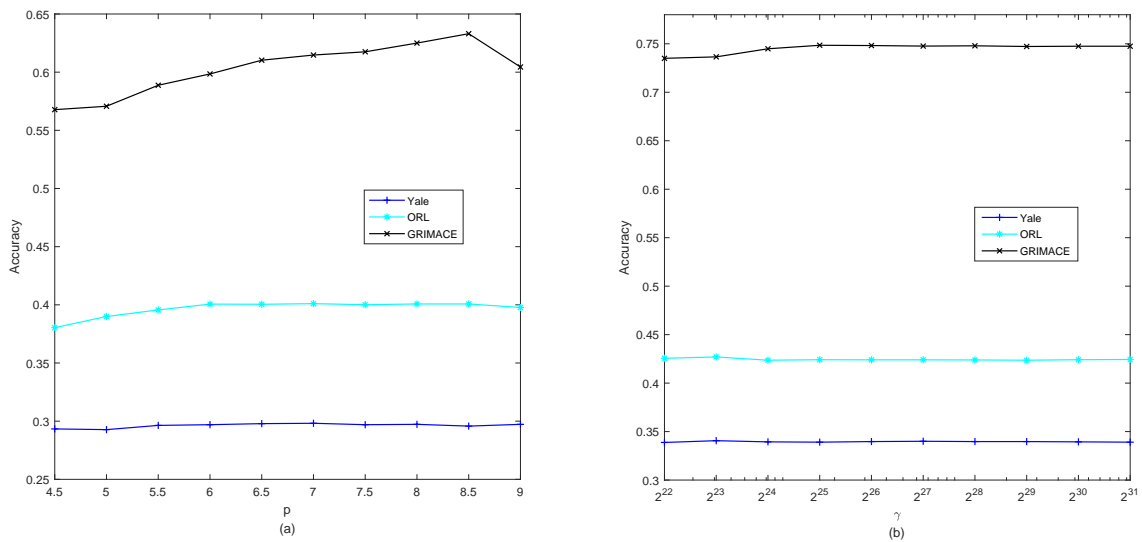}
	\caption{Accuracy with different parameters on three datasets: (a) $p$ for FWNMF and (b) $\gamma$ for ERWNMF. }
	\label{fig.2}
\end{figure*}

\begin{figure*}[htbp]
	\centering
	\includegraphics[scale=0.45]{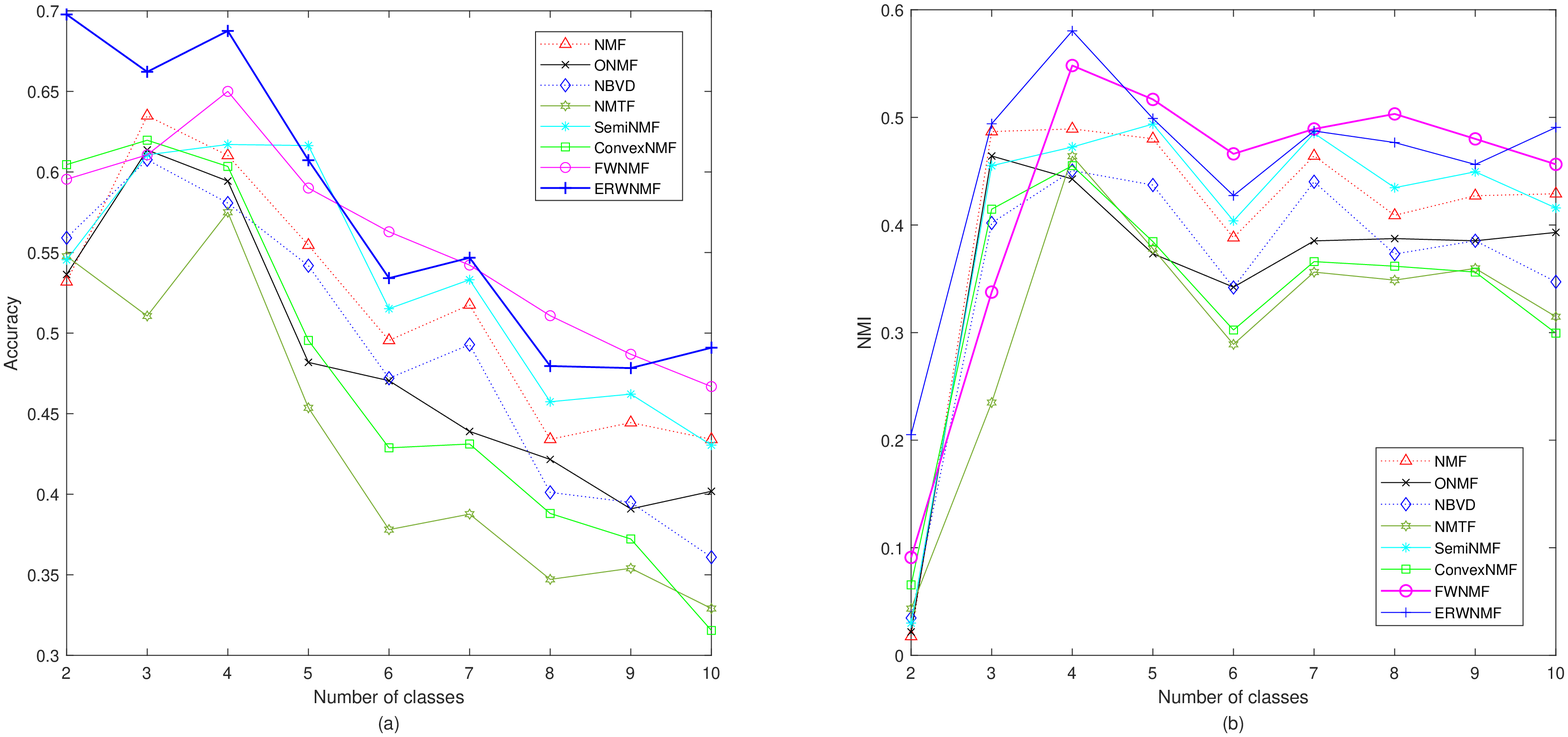}
	\caption{Clustering performance versus cluster number on Yale dataset: (a) accuracy and (b) NMI.}
	\label{fig.3} 
\end{figure*}

\begin{figure*}[htbp]
	\centering
	\includegraphics[scale=0.45]{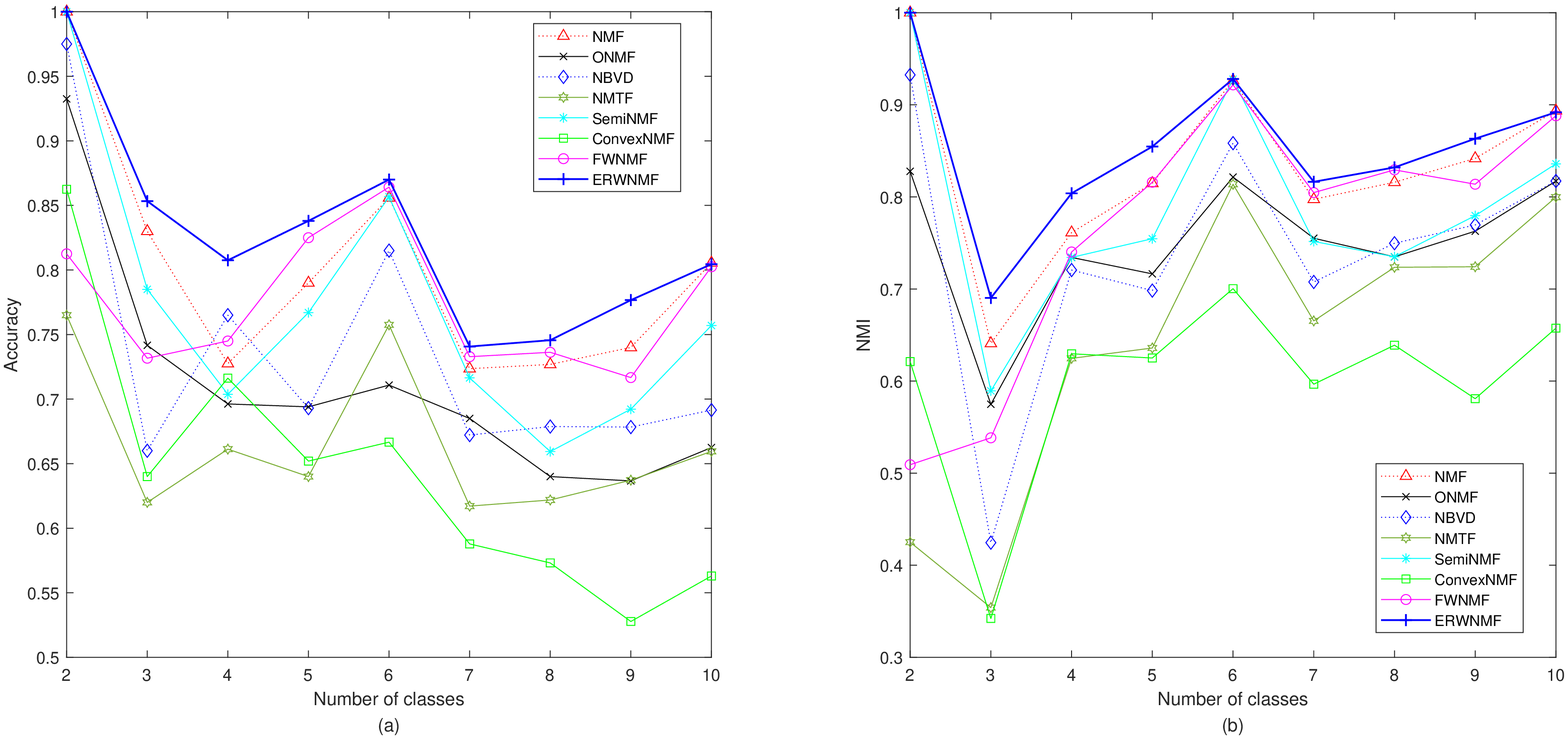}
	\caption{Clustering performance versus cluster number on ORL dataset: (a) accuracy and (b) NMI.}
	\label{fig.4}
\end{figure*}

\begin{figure*}[htbp]
	\centering
	\includegraphics[scale=0.45]{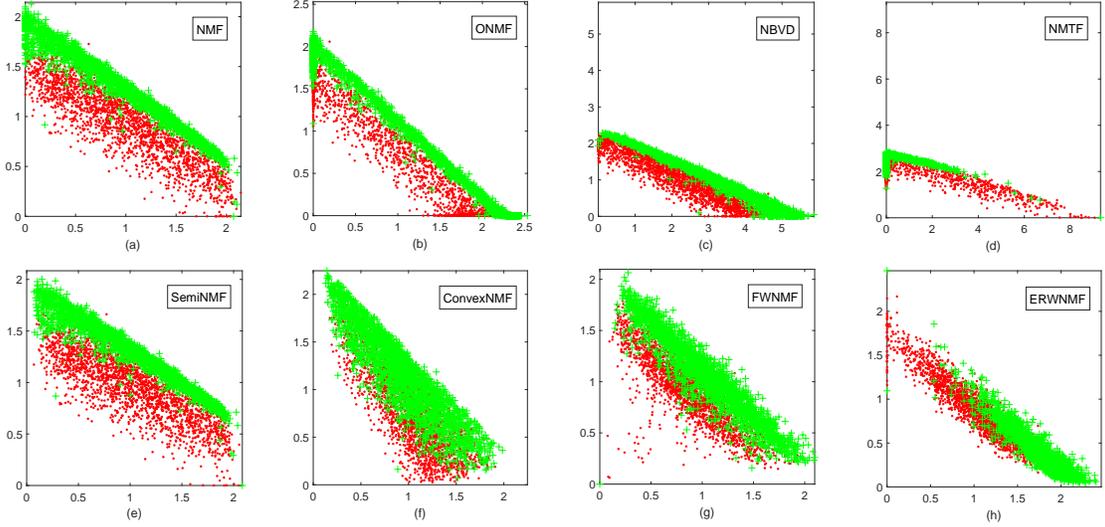}
	\caption{Visualization of different methods on Dimdata dataset: (a) NMF, (b) ONMF, (c) NBVD, (d) NMTF, (e) SemiNMF, (f) ConvexNMF, (g) FWNMF and (h) ERWNMF, where plus sign ($+$) and dot ($\cdot$) represent two categories, respectively.}
	\label{fig.5}
\end{figure*}

\begin{figure*}[htbp]
	\centering
	\includegraphics[scale=0.45]{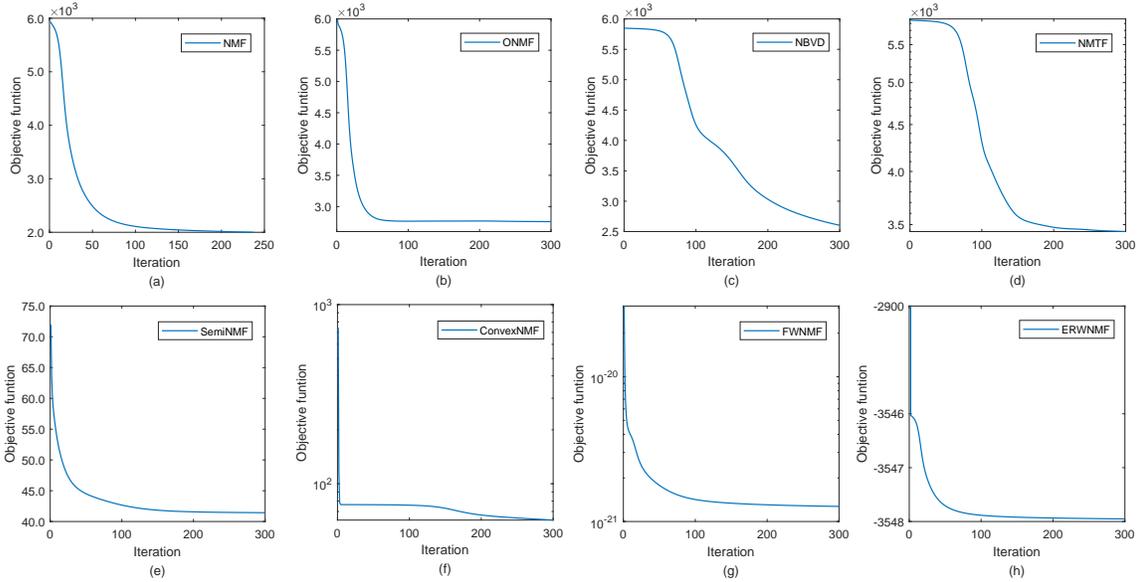}
	\caption{Objective functions on the Yale dataset are obtained by (a) NMF, (b) ONMF, (c) NBVD, (d) NMTF, (e) SemiNMF, (f) ConvexNMF, (g) FWNMF, and (h) ERWNMF.}
	\label{fig.6}
\end{figure*}

\begin{figure*}[htbp]
	\centering
	\includegraphics[scale=0.45]{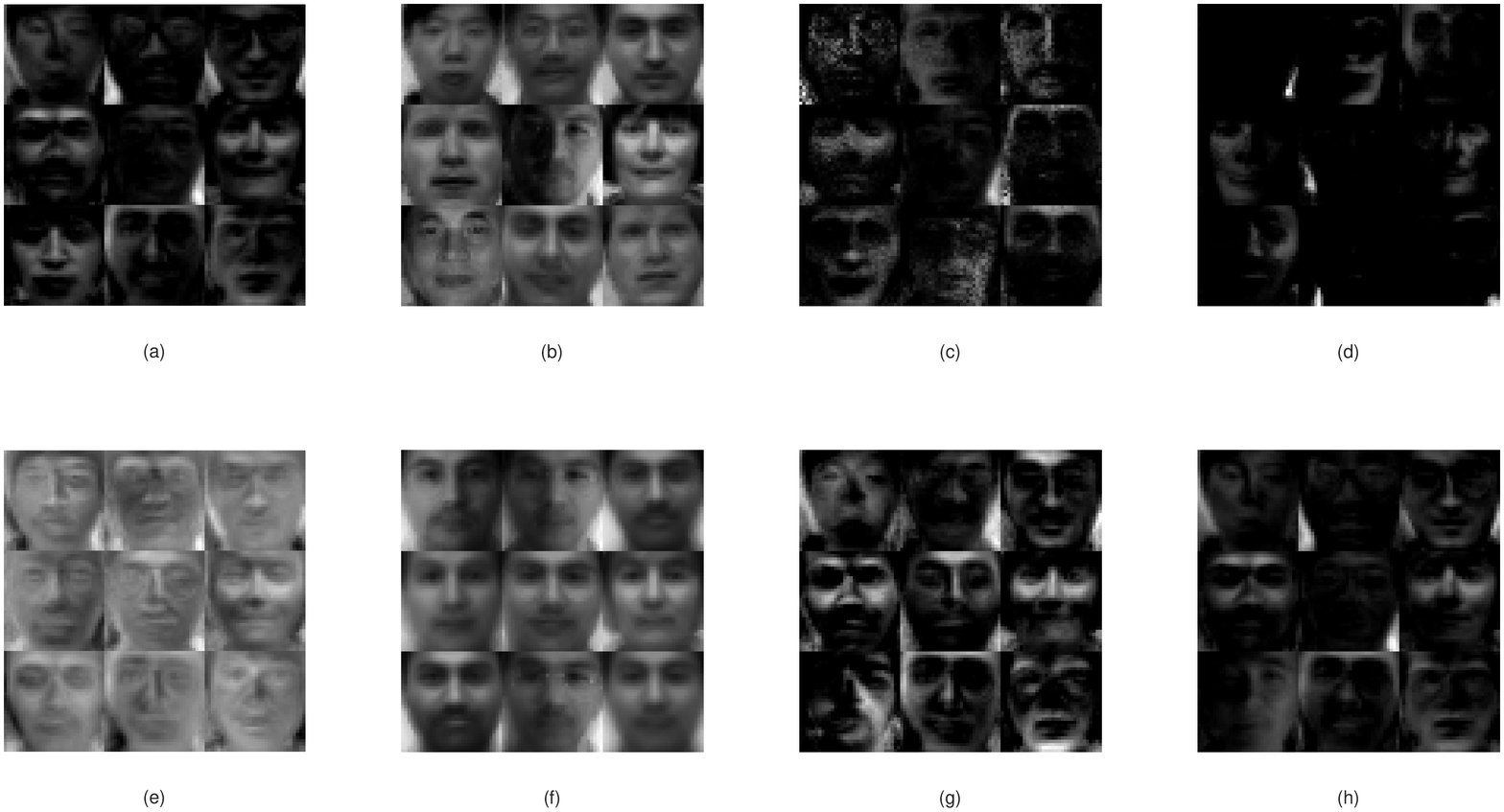}
	\caption{Some base images with respect to Yale dataset, in which summation of every column is normalized to 1.
    They are obtained by (a) NMF, (b) ONMF, (c) NBVD, (d) NMTF, (e) SemiNMF, (f) ConvexNMF, (g) FWNMF, and (h) ERWNMF.}
	\label{fig.7}
\end{figure*}
\subsection{Experimental results}
\subsubsection{Interpretation to weights}
We utilize the following experiment to verify the efficiency of our methods. To simulate a camera with some sensors destroyed, we replaced the area of $12\times 12$ on the image in the Yale dataset with a random number uniformly distributed on [0 1]. We visualized the first six corrupted images of the dataset and reshaped $W$ obtained by Eqs. (\ref{eq16}) and (\ref{eq27}) in Fig. \ref{fig.1}. For FWNMF and ERWNMF, the parameters $p$ and $\gamma$ are chosen in the order of $\{4, 4.5, 5, 5.5, 6, 6.5 \}$ and $\{2^2,2^3,2^4,2^5,2^6,2^7\}$, respectively. As shown, for the corrupted pixels, both of the proposed methods with different hyperparameters can obtain right weights, which approach zero. However, for the other pixels, different hyperparameters lead to different results, whose effects are further investigated as follows. 

\subsubsection{Hyperparameter robustness} 
We discuss the hyperparameter robustness of FWNMF and ERWNMF on three datasets: Yale, ORL, and GRIMACE. We repeat the representation learning and clustering procedures 20 times to avoid the effect of initialization because NMF has no sole solution. Fig. \ref{fig.2} shows the changes in clustering accuracy when the hyperparameters of FWNMF and ERWNMF are changed, where the parameters $p$ of FWNMF and $\gamma$ of ERWNMF are chosen in $\{0.5*i,\,i=9,10,...,17,18\}$ and $\{2^i,\,i=22,23,...,31\}$, respectively. It can be seen that on the Yale and ORL datasets, the parameter $p$ performs well from 6 to 9 with less fluctuation in performance from 4.5 to 6, while $\gamma$ maintains a relatively stable state, indicating that accuracy is not sensitive to $\gamma$ on this dataset. On the GRIMACE dataset, the parameters $p$ and $\gamma$ have a slight fluctuation, showing a growing state overall. We can draw the conclusion that our methods are stable over a large range of hyperparameter values.

\subsubsection{Clustering results} 
We study the relationship between clustering performance and reduced dimensionality on the Yale and ORL datasets. First, we randomly select $k$ categories from the dataset as a subset. We apply the methods to decompose this subset to obtain a new data representation. We set the reduced dimensionality to be the same as the number of clusters. Then, k-means is applied to the new representation for clustering. We compare the obtained clustering results with the labels and compute the accuracy and NMI. For all the methods, the hyperparameters are set to the values at which each method can achieve its best results. For FWNMF, the parameter $p$ is selected in $\{0.5*i,\,i=3,4,...,59,60\}$. For ERWNMF, the parameter $\gamma$ is selected in$\{2^i,\,i=1,2,...,31\}$. We still repeat the above experiments 20 times for a credible comparison. Figs. \ref{fig.3} and \ref{fig.4} show the accuracy and NMI versus the number of clusters. 
On the Yale dataset, one of the two proposed methods generally obtains the best evaluation; however, none of the two is always the best. The standard NMF and SemiNMF are slightly lower, while the results for ONMF, NBVD, and NMTF are very similar. The results of ConvexNMF are generally the worst. On the ORL dataset, ERWNMF has the best performance, and FWNMF is the second-best in most cases.

Tables \ref{tab2} and \ref{tab3} describe the clustering results on the whole datasets, where the number of clusters is fixed in Table \ref{tab1}. We still repeat the process 20 times to obtain an averaged result. The best result is almost always obtained by FWNMF and ERWNMF for all datasets, with only one exception: SemiNMF obtains the best NMI on the Yale dataset, while FWNMF and ERWNMF rank second and third, respectively. On the whole, ERWNMF performs better than FWNMF, especially for accuracy, in which the former is always the best; and even for NMI, ERWNMF also generally obtains better results than FWNMF.
\subsubsection{Nonimage experiments}
We explore the application of FWNMF and ERWNMF on nonimage datasets. The experiment was performed on the Dimdata dataset, containing 4192 samples with 14 features. We compress the original data in two-dimensional space since it only includes two clusters, and the new representation is visualized in Fig. \ref{fig.5}. Viewed with the naked eye, it is easy to observe that the new representations obtained by ERWNMF have good separability if considering using a center-distance-based clustering method such as k-means. This explains why ERWNMF obtains an excellent result in Tables \ref{tab2} and \ref{tab3} regardless of accuracy and NMI.

\subsubsection{Convergence study}
Through theoretical analysis, we found that FWNMF and ERWNMF are monotonically decreasing, but due to the nonconvexity of their objective functions (similar to other NMF methods), it is not guaranteed to be strictly convergent \cite{lin2007convergence}. Therefore, we study their convergence speed. We conducted experiments on the Yale dataset and plotted the objective function in Fig. \ref{fig.6}. We can see that both FWNMF and ERWNMF converge very fast. In particular, ERWNMF can achieve fast convergence with fewer iterations, showing its superior performance. 

We also show the base images obtained by all the methods on the Yale dataset in Fig. \ref{fig.7}. It can be seen that NMTF obtains the sparsest base images; ERWNMF is the second sparsest; ONMF, SemiNMF, and ConvexNMF identify more global faces; and the other methods present similar sparse base images that are difficult to distinguish from each other.

\section{Conclusion}
This paper proposed two new methods to obtain feature representations on the subspace of the original data by introducing two types of adaptive weights for the features. These adaptive weights are interpretable and computable and can be popularized in many existing NMF methods. Our proposed method overcomes the limitations of previous methods and extracts the key features of the original data. We conducted experiments on several datasets to confirm their superiority.
However, both of the proposed methods require an additional hyperparameter to control the certainty of the weights. We plan to develop an auto-adjustment strategy for this hyperparameter in the future.

\section*{Appendix}
This is an example to explain Eq. (\ref{eq:minW}). 
\begin{equation}
	\begin{aligned}
		\min~\{5, 1, 3\}= min&~(5W_1+1W_2+3W_3)\\ 
		s.~t.&~ W_1\geq0, W_2\geq0, W_3\geq0\\
		&~W_1+W_2+W_3=1\nonumber
	\end{aligned}
\end{equation}
It is very easy to verify that the solution is that $W_1=0$, $W_2=1$, and $W_3=0$, in which $W_2=1$ corresponds to the minimum value of \{5, 1, 3\}.
\section*{Acknowledgment}
\indent The authors would like to thank the editor and anonymous reviewers for their constructive comments and suggestions.

\bibliographystyle{IEEEtran}
\bibliography{IEEEabrv,cankaowen}

\begin{IEEEbiography}[{\includegraphics[width=1in, height=1.25in, clip,keepaspectratio]{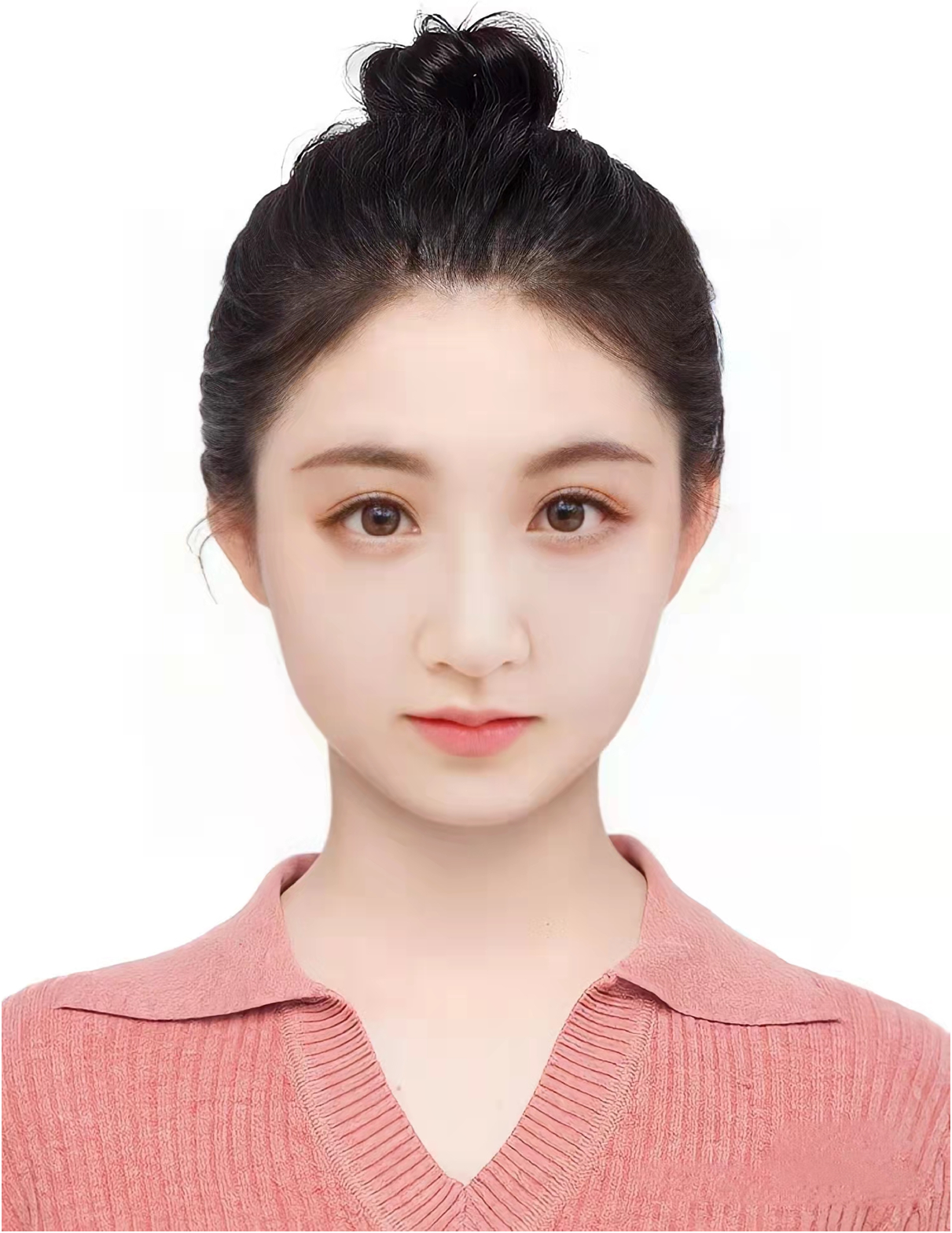}}]{Junhang Li} is pursuing her master's degree in the School of Medicine and Biological Engineering at Northeastern University and received her bachelor's degree from Shenyang Ligong University in 2017.
Her research interests include machine learning and nonnegative matrix factorization.
\end{IEEEbiography}
\begin{IEEEbiography}[{\includegraphics[width=1in, height=1.25in, clip,keepaspectratio]{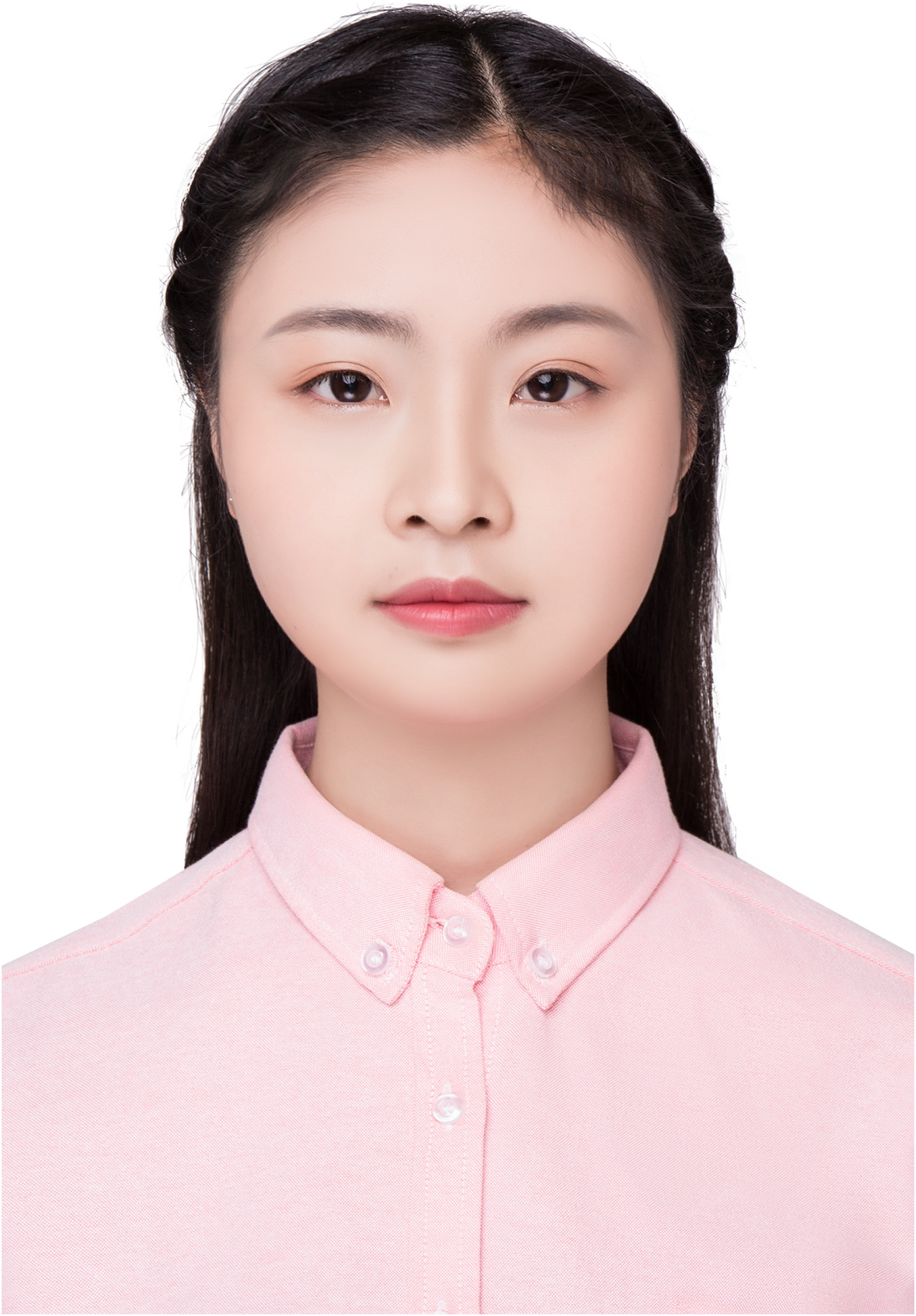}}]{Jiao Wei} is pursuing her M.S. degree in the School of Medical and Bioinformatics Engineering at Northeastern University and received her B.S. degree from Qufu Normal University in 2014. Her research interests include machine learning and nonnegative matrix decomposition algorithms.
\end{IEEEbiography}
\begin{IEEEbiography}[{\includegraphics[width=1in, height=1.25in, clip,keepaspectratio]{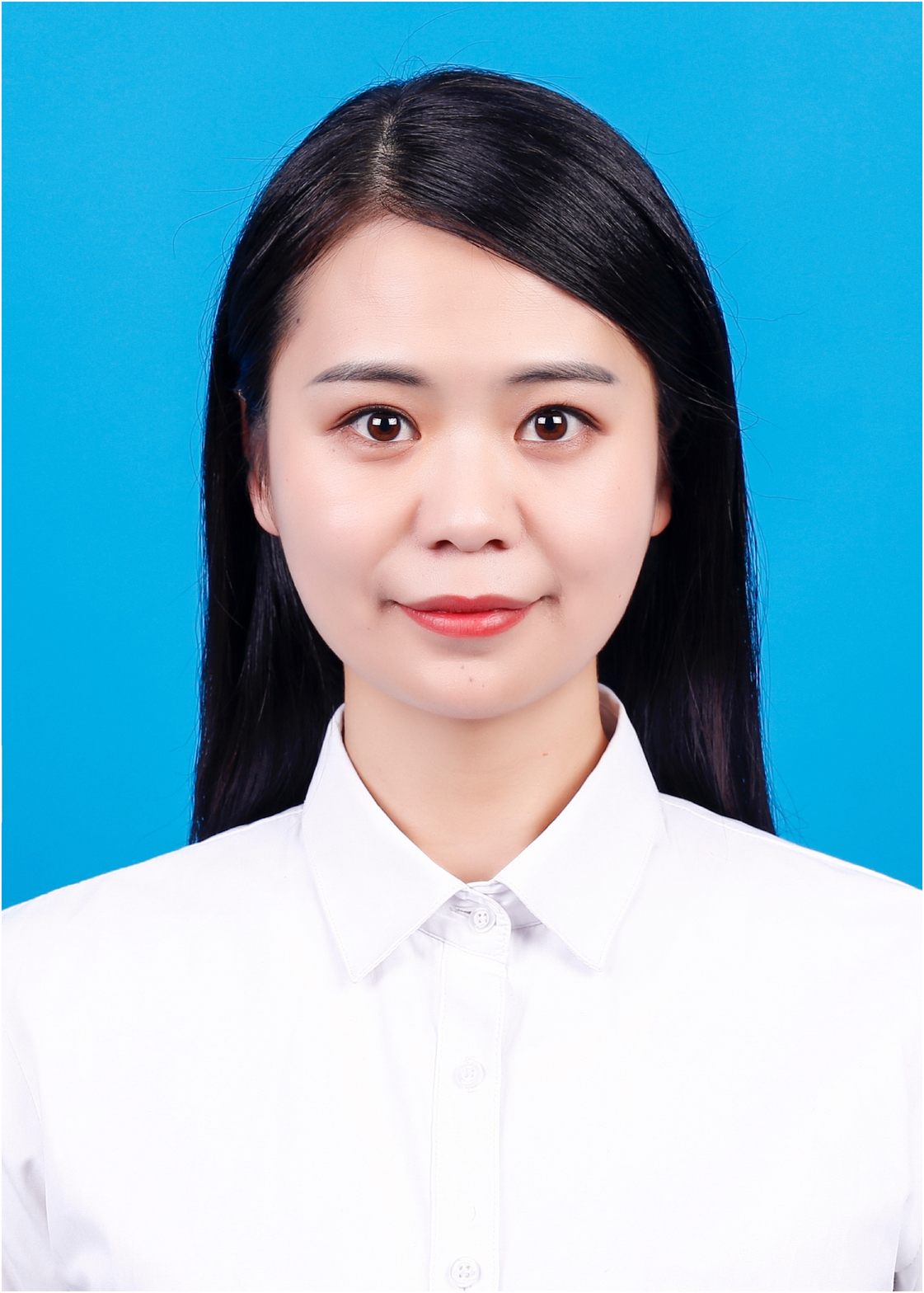}}]{Can Tong} is a Ph.D. student and received a bachelor's degree in mathematics and applied mathematics from 
Northeastern University 
in 2014 and a master's degree in computational mathematics from Northeastern University in 2018. Her primary research interests include basic algorithm theory and acceleration methods of machine learning.
\end{IEEEbiography}

\begin{IEEEbiography}[{\includegraphics[width=1in,height=1.25in,clip,keepaspectratio]{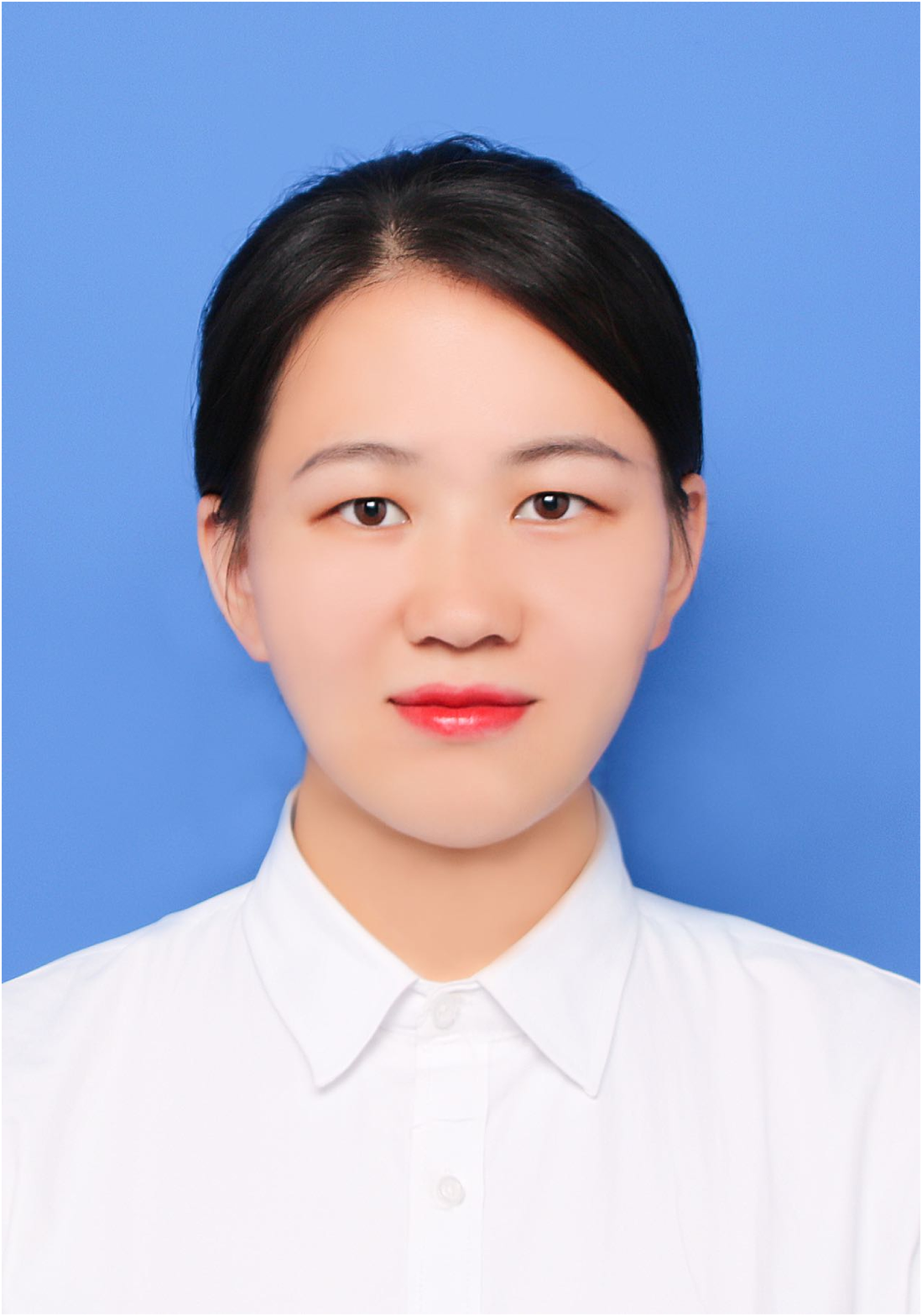}}]{Tingting Shen} is pursuing her M.S. degree in the School of Medical and Bioinformatics Engineering at Northeastern University and received her B.S. degree from Hefei University in 2016. Her research interests are machine learning and deep learning.
\end{IEEEbiography}

\begin{IEEEbiography}[{\includegraphics[width=1in, height=1.25in, clip,keepaspectratio]{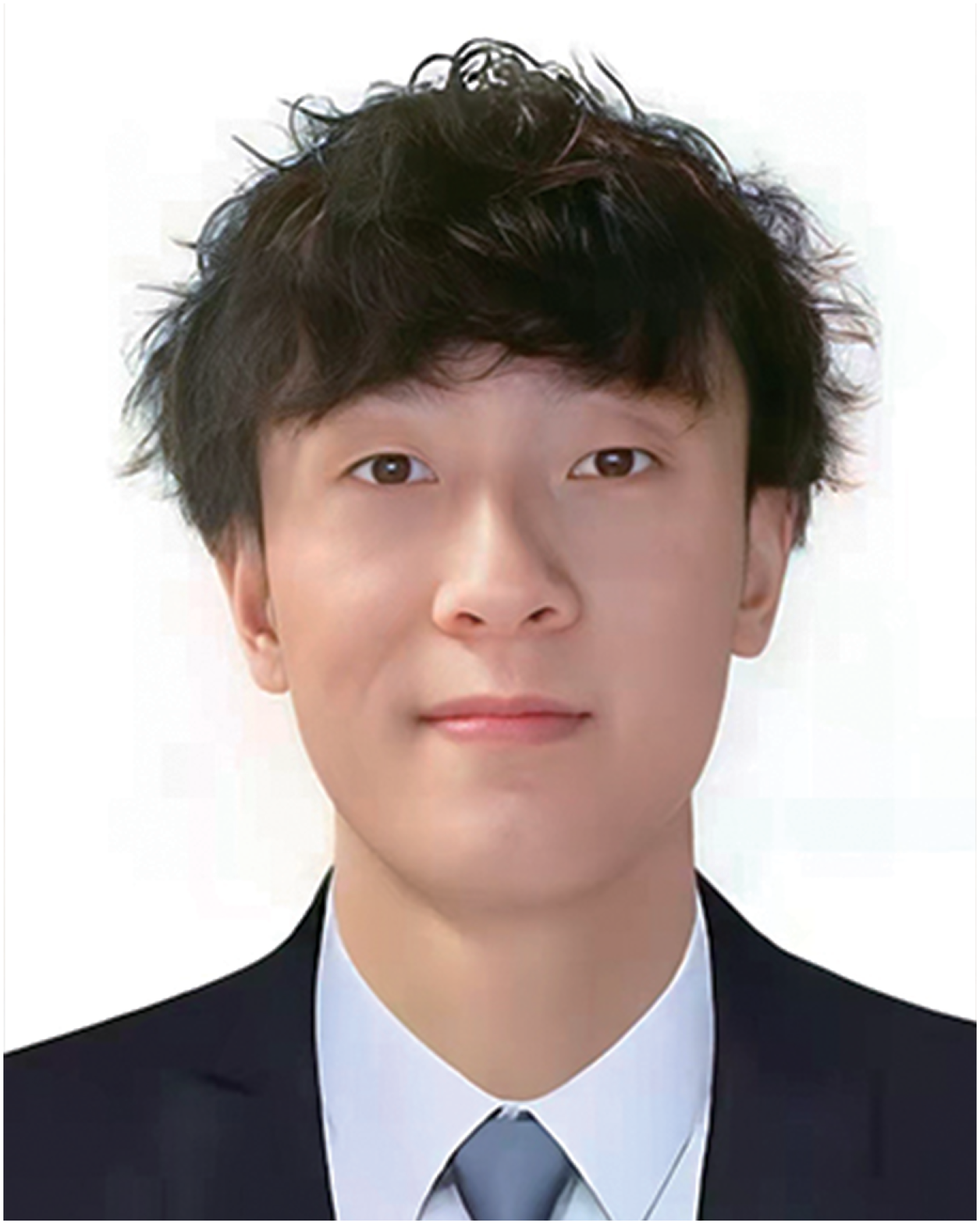}}]{Yuchen Liu} is pursuing his bachelor's degree in the College of Medicine and Biological Information Engineering at Northeastern University. His research interests include machine learning and ensemble learning.
\end{IEEEbiography}

\begin{IEEEbiography}[{\includegraphics[width=1in,height=1.25in,clip,keepaspectratio]{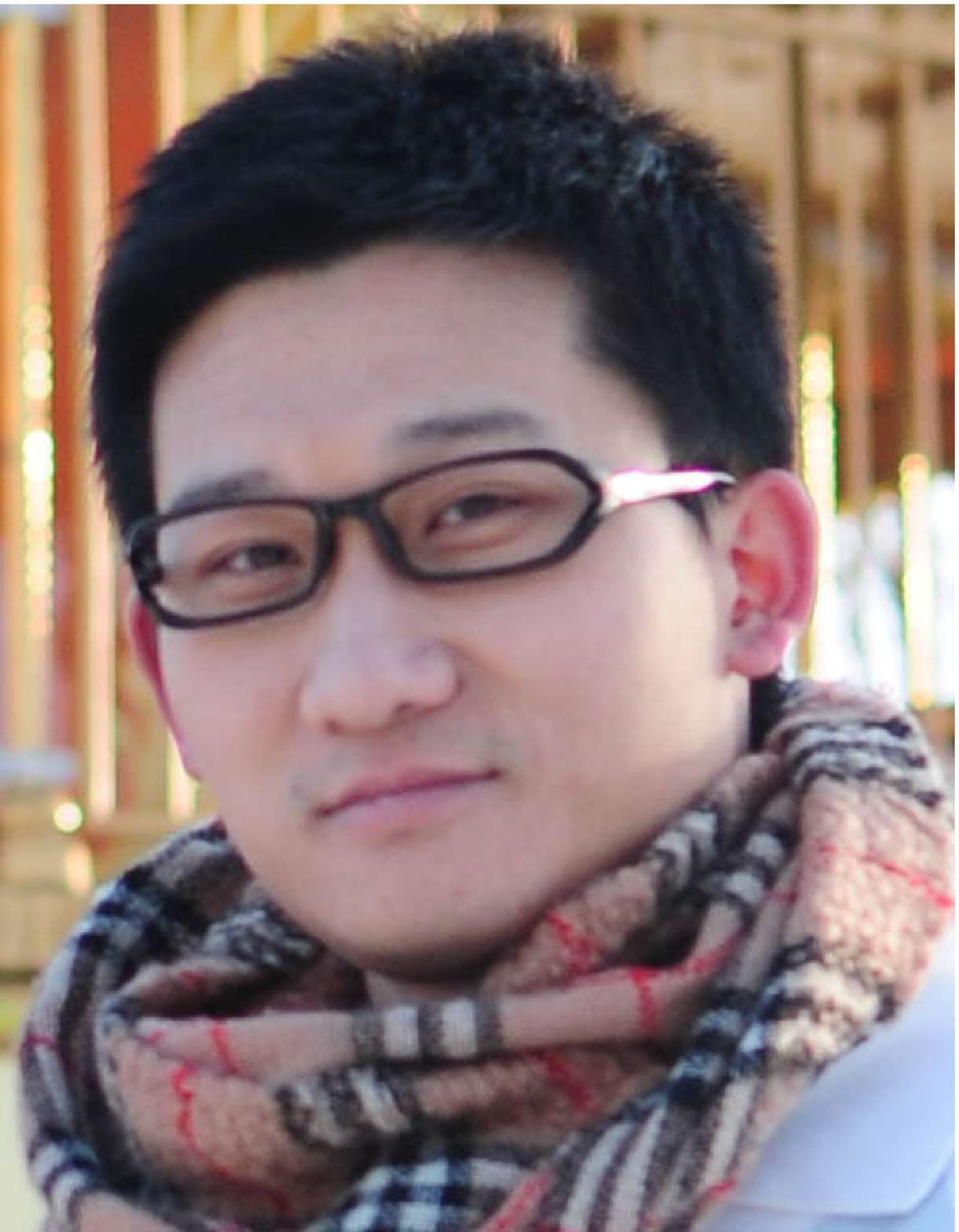}}]{Chen Li} received a bachelor's degree in engineering from the University of Science and Technology Beijing in 2008, a master's degree in science from Northeast Normal University in 2011, and a doctorate in engineering from the University of Siegen in Germany in 2016. From 2016 to 2017, he was engaged in postdoctoral research at the University of Siegen and Mainz in Germany. From 2011 to 2017, he worked as a part-time teaching assistant, lecturer, and master tutor at the University of Siegen and Mainz. Since April 2017, he has been engaged in scientific research and teaching in the School of Medicine and Bioinformatics Engineering of Northeastern University.
\end{IEEEbiography}

\begin{IEEEbiography}[{\includegraphics[width=1in, height=1.25in, clip,keepaspectratio]{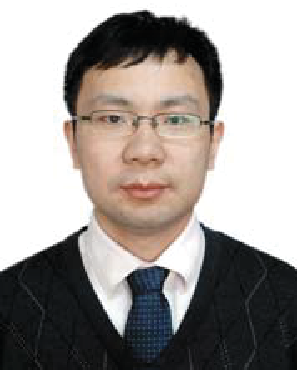}}]{Shouliang Qi} is an associate professor at Northeastern University, China. After receiving a philosophy doctorate from Shanghai Jiao Tong University in 2007, he joined the GE Global Research Center and was responsible for designing an innovative magnetic resonance imaging (MRI) system. In 2014-2015, he worked as a visiting scholar at Eindhoven University of Technology and Epilepsy Center Kempenhaeghe, The Netherlands. In recent years, he has been conducting productive studies in advanced MRI technology, intelligent medical imaging computing, modeling, and the convergence of Nano-Bio-Info-Cogn (NBIC). He has published more than 60 papers in peer-reviewed journals and international conferences. He has won many academic awards including the Chinese Excellent PH. D. Dissertation Nomination Award and the Award for Outstanding Achievement in Scientific Research from the Ministry of Education.
\end{IEEEbiography}

\begin{IEEEbiography}[{\includegraphics[width=1in, height=1.25in, clip,keepaspectratio]{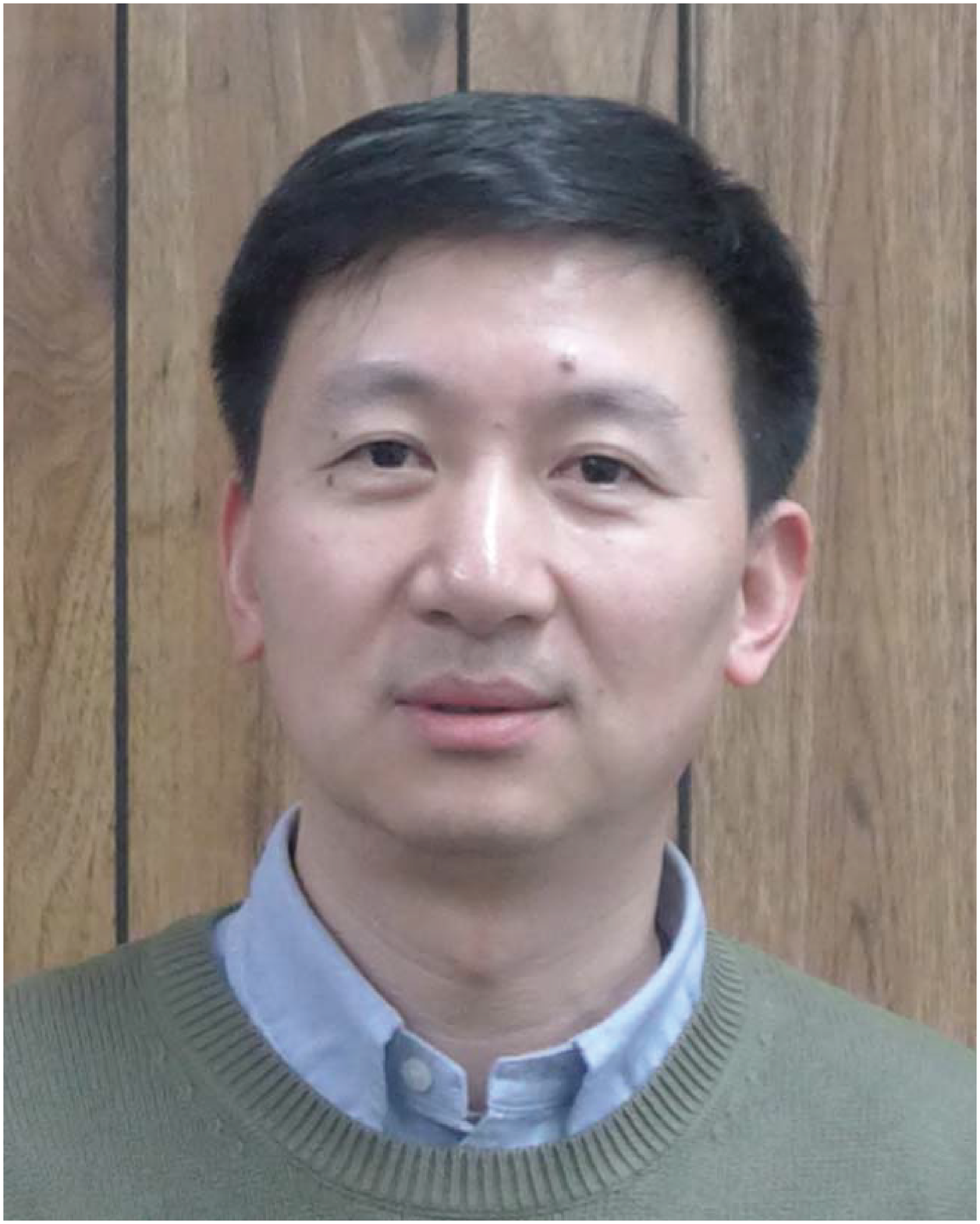}}]{Yudong Yao} (S'88-M'88-SM'94-F'11) received B.Eng. and M.Eng. degrees from Nanjing University of Posts and Telecommunications, Nanjing, in 1982 and 1985, respectively, and a Ph.D. degree from Southeast University, Nanjing, in 1988, all in electrical engineering. He was a visiting student at Carleton University, Ottawa, in 1987 and 1988. Dr. Yao has been with Stevens Institute of Technology, Hoboken, New Jersey, since 2000 and is currently a professor and department director of electrical and computer engineering. He is also a director of Stevens' Wireless Information Systems Engineering Laboratory (SLAB). Previously, from 1989 to 2000, Dr. Yao worked for Carleton University, Spar Aerospace, Ltd., Montreal, and Qualcomm, Inc., San Diego. He has been active in a nonprofit organization, WOCC, Inc., promoting wireless and optical communications research and technical exchange. He served as WOCC president (2008-2010) and chairman of the board of trustees (2010-2012). His research interests include wireless communications and networking, cognitive radio, machine learning, and big data analytics. He holds one Chinese patent and thirteen U.S. patents. Dr. Yao was an associate editor of IEEE Communications Letters (2000-2008) and IEEE Transactions on Vehicular Technology (2001-2006) and an editor for IEEE Transactions on Wireless Communications (2001-2005). He was elected an IEEE Fellow in 2011 for his contributions to wireless communications systems and was an IEEE ComSoc Distinguished Lecturer (2015-2018). In 2015, he was elected a Fellow of the National Academy of Inventors.
\end{IEEEbiography}
\begin{IEEEbiography}[{\includegraphics[width=1in, height=1.25in, clip,keepaspectratio]{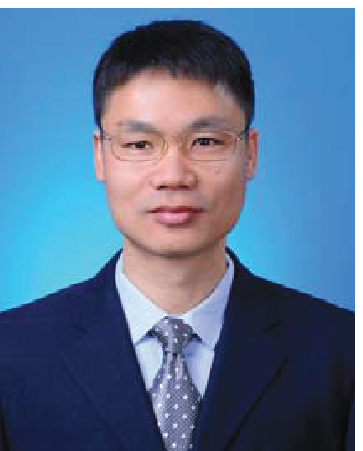}}]{Yueyang Teng} received his bachelor's and master's degrees from the Department of Applied Mathematics, Liaoning Normal University and Dalian University of Technology, China, in 2002 and 2005, respectively. From 2005 to 2013, he was a software engineer at Neusoft Position Medical Systems Co., Ltd. In 2013, he received his doctorate in Computer Software and Theory from Northeastern University. From 2013 to the present, he has been an associate professor in the School of Medical and Bioinformatics Engineering, Northeastern University. His research interests include image processing and machine learning.
\end{IEEEbiography}

\end{document}